\documentclass[letterpaper, 10 pt, conference]{ieeeconf}  

\IEEEoverridecommandlockouts                              

\usepackage{amssymb}
\usepackage{bm}
\usepackage{booktabs} 
\usepackage{xspace}
\usepackage{graphicx}
\usepackage{amsmath}
\usepackage{hyperref}

\newtheorem{theorem}{Theorem}
\newtheorem{lemma}{Lemma}
\newtheorem{corollary}{Corollary}
\usepackage[ruled]{algorithm2e} 

\DeclareMathOperator*{\argmax}{arg\,max}
\DeclareMathOperator*{\argmin}{arg\,min}

\newcommand{\vectw}[2]{(#1,\: #2)}
\newcommand{\vecth}[3]{(#1,\: #2,\: #3)}
\newcommand{\learn}{\left(\pi\right)}
\newcommand{\expert}{\left(\pi^*\right)}

\title{\LARGE \bf
ADAPS: Autonomous Driving Via Principled Simulations
}

\author{Weizi Li$^{1}$, David Wolinski$^{1}$, and Ming C. Lin$^{1,2}$
\thanks{$^{1}$W. Li, D. Wolinski, M. Lin are with the Department of Computer Science,
        University of North Carolina at Chapel Hill, NC, USA
        {\tt\small \{weizili,dwolinsk,lin\}@cs.unc.edu}}%
\thanks{$^{2}$M. Lin is now with the Department of Computer Science, University of Maryland at College Park,
        MD, USA
        {\tt\small lin@cs.umd.edu}}%
}

\begin{document}

\maketitle
\thispagestyle{empty}
\pagestyle{empty}

\begin{abstract}
Autonomous driving has gained significant advancements in recent years. However, obtaining a robust control policy for driving remains challenging as it requires training data from a variety of scenarios, including rare situations (e.g., accidents), an effective policy architecture, and an efficient learning mechanism. We propose ADAPS for producing robust control policies for autonomous vehicles. ADAPS consists of two simulation platforms in generating and analyzing accidents to automatically produce labeled training data, and a memory-enabled hierarchical control policy. Additionally, ADAPS offers a more efficient online learning mechanism that reduces the number of iterations required in learning compared to existing methods such as DAGGER~\cite{ross2011reduction}. We present both theoretical and experimental results. The latter are produced in simulated environments, where qualitative and quantitative results are generated to demonstrate the benefits of ADAPS.
\end{abstract}


\section{Introduction}
\label{sec:intro}

Autonomous driving consists of many complex sub-tasks that consider the dynamics of an environment and often lack accurate definitions of various driving behaviors. These characteristics lead to conventional control methods to suffer subpar performance on the task~\cite{ratliff2009learning,silver2010learning}. However, driving and many other tasks can be easily demonstrated by human experts. This observation inspires \textit{imitation learning}, which leverages expert demonstrations to synthesize a controller. 

While there are many advantages of using imitation learning, it also has drawbacks. For autonomous driving, the most critical one is \textit{covariate shift}, meaning the training and test distributions are different. This could lead autonomous vehicles (AVs) to accidents since a learned policy may fail to respond to unseen scenarios including those dangerous situations that do not occur often. 

In order to mitigate this issue, the training dataset needs to be augmented with more expert demonstrations covering a wide spectrum of driving scenarios---especially ones of significant safety threats to the passengers---so that a policy can learn how to recover from its own mistakes. This is emphasized by Pomerleau~\cite{pomerleau1989alvinn}, who synthesized a neural network based controller for AVs: ``the network must not solely be shown examples of accurate driving, but also how to recover (i.e. return to the road center) once a mistake has been made.'' 

Although critical, obtaining recovery data from accidents in the physical world is impractical due to the high cost of a vehicle and potential injuries to both passengers and pedestrians. In addition, even one managed to collect accident data, human experts are usually  needed to label them, which is inefficient and may subject to judgmental errors~\cite{ross2013learning}.

These difficulties naturally lead us to the virtual world, where accidents can be simulated and analyzed~\cite{Survey:2019}. We have developed ADAPS (\textbf{A}utonomous \textbf{D}riving Vi\textbf{a} \textbf{P}rincipled \textbf{S}imulations) to achieve this goal. ADAPS consists of two simulation platforms and a memory-enabled hierarchical control policy based on deep neural networks (DNNs). The first simulation platform, referred to as \textit{SimLearner}, runs in a 3D environment and is used to test a learned policy, simulate accidents, and collect training data. The second simulation platform, referred to as \textit{SimExpert}, acts in a 2D environment and serves as the ``expert'' to analyze and resolve an accident via \textit{principled simulations} that can plan alternative safe trajectories for a vehicle by taking its physical, kinematic, and geometric constraints into account.


Furthermore, ADAPS represents a more efficient online learning mechanism than existing methods such as DAGGER~\cite{ross2011reduction}. This is useful consider learning to drive requires iterative testing and update of a control policy. Ideally, we want to obtain a robust policy using minimal iterations since one iteration corresponds to one incident. This would require the generation of training data at each iteration to be \textit{accurate}, \textit{efficient}, and \textit{sufficient} so that a policy can gain a large improvement going into the next iteration. ADAPS can assist to achieve this goal.

The \textbf{main contributions} of this research are specifically: (1) The accidents generated in \textit{SimLearner} will be analyzed by \textit{SimExpert} to produce alternative safe trajectories. (2) These trajectories will be automatically processed to generate a large number of annotated and segmented training data. Because \textit{SimExpert} is parameterized and has taken the physical, kinematic, and geometric constraints of a vehicle into account (i.e., principled), the resulting training examples are more heterogeneous than data collected via running a learned policy multiple times and are more effective than data collected through random sampling. (3) We present both theoretical and experimental results to demonstrate that ADAPS is an efficient online learning mechanism.

The Appendix, which contains supporting material, can be found at~\url{http://gamma.cs.unc.edu/ADAPS/}.

\section{Related Work}
\label{sec:related}
We sample previous studies that are related to each aspect of our framework and discuss the differences within.

\textbf{Autonomous Driving}. Among various methods to plan and control an AV~\cite{schwarting2018planning}, we focus on end-to-end imitation learning as it can avoid manually designed features and lead to a more compact policy compared to conventional mediation perception approaches~\cite{chen2015deepdriving}. The early studies done by Pomerleau~\cite{pomerleau1989alvinn} and LeCun et al.~\cite{LeCun2006off} have shown that neural networks can be used for an AV to achieve lane-following and off-road obstacle avoidance. Due to the advancements of deep neural networks (DNNs), a number of studies have emerged~\cite{Bojarski2016,Xu2017end,pan2017agile,codevilla2017end}. While significant improvements have been made, these results mainly inherit normal driving conditions and restrict a vehicle to the  lane-following behavior~\cite{codevilla2017end}. Our policy, in contrast, learns from accidents and enables a vehicle to achieve \textit{on-road} collision avoidance with both static and dynamic obstacles.

\textbf{Hierarchical Control Policy}. There have been many efforts in constructing a hierarchical policy to control an agent at different stages of a task~\cite{barto2003recent}.
Example studies include the options framework~\cite{sutton1999between} and transferable motor skills~\cite{konidaris2012robot}.
When combined with DNNs, the hierarchical approach has been adopted for virtual characters to learn locomotion tasks~\cite{levine2013guided}.
In these studies, the goal is to discover a hierarchical relationship from complex sensorimotor behaviors. We apply a hierarchical and memory-enabled policy to autonomous driving based on multiple DNNs. Our policy enables an AV to continuously categorize the road condition as safe or dangerous, and execute corresponding control commands to achieve accident-free driving.

\textbf{Generative Policy Learning}. Using \textit{principled simulations} to assist learning is essentially taking a \textit{generative model} approach. Several studies have adopted the same philosophy to learn (near-)optimal policy, examples including function approximations~\cite{gordon1995stable}, Sparse Sampling~\cite{kearns2002sparse}, and Fitted Value Iteration~\cite{szepesvari2005finite}. These studies leverage a generative model to \textit{stochastically} generate training samples. The emphasize is to simulate the feedback from an environment instead of the dynamics of an agent assuming the \textit{reward function is known}. Our system, on the other hand, does not assume any reward function of a driving behavior but models the physical, kinematic, and geometric constraints of a vehicle, and uses simulations to plan their trajectories w.r.t. environment characteristics. In essence, our method learns from expert demonstrations rather than self-exploration~\cite{lin1992self} as of the previous studies.

%

\section{Preliminaries}
\label{sec:prelim}

Autonomous driving is a \textit{sequential prediction} and \textit{controlled} (SPC) task, for which a system must predict a sequence of control commands based on inputs that depend on past predicted control commands. Because the control and prediction processes are intertwined, SPC tasks often encounter \textit{covariate shift}, meaning the training and test distributions vary. In this section, we will  first introduce notation and definitions to formulate an SPC task and then briefly discuss its existing solutions.

\subsection{Notation and Definitions}
\label{sec:notation}
The problem we consider is a $ T $-step control task. Given the observation $ \phi=\phi(s) $ of a state $ s $ at each step $ t \in [\![1, T]\!] $, the goal of a learner is to find a policy $ \pi \in \Pi $ such that its produced action $ a=\pi(\phi)$ will lead to the minimal cost:

\vspace{-0.8em}

\begin{equation}
\hat{\pi} = \argmin_{\pi \in \Pi} \sum_{t=1}^{T} C\left(s_{t}, a_{t}\right),
\label{eq:cost1}
\end{equation}

\noindent where $ C\left(s,a\right) $ is the expected immediate cost of performing $ a $ in  $ s $. For many tasks such as driving, we may not know the true value of $ C$. So, we instead minimize the observed surrogate loss $ l(\phi,\pi, a^*) $, which is assumed to upper bound $ C $, based on the approximation of the learner's action $ a=\pi(\phi) $ to the expert's action $ a^*=\pi^*(\phi) $. We denote the distribution of observations at $ t $ as $ d_{\pi}^t $, which is the result of executing $ \pi $ from $ 1 $ to $ t-1 $. Consequently, $ d_{\pi}=\frac{1}{T}\sum_{t=1}^{T}d_{\pi}^t $ is the average distribution of observations by executing $ \pi $ for $ T $ steps. Our goal is to solve an SPC task by obtaining $ \hat{\pi} $ that minimizes the observed surrogate loss under its own induced observations w.r.t. expert's actions in those observations:

\vspace{-0.8em}

\begin{equation}
\hat{\pi} = \argmin_{\pi \in \Pi} \mathbb{E}_{\phi \sim d_{\pi}, a^* \sim \pi^*(\phi)} \left[l\left(\phi,\pi, a^*\right)\right].
\end{equation}

\noindent We further denote $ \epsilon =  \mathbb{E}_{\phi \sim d_{\pi^*}, a^* \sim \pi^*(\phi)} \left[l\left(\phi,\pi, a^*\right)\right] $ as the expected loss under the training distribution induced by the expert's policy $ \pi^* $, and the cost-to-go over $ T $ steps of $ \hat{\pi} $ as $ J\left(\hat{\pi}\right) $ and of $ \pi^* $ as $ J\left(\pi^*\right) $. It has been shown that by simply treating expert demonstrations as i.i.d. samples the discrepancy between $ J\left(\hat{\pi}\right) $ and $ J\left(\pi^*\right) $ is $ \mathcal{O}(T^2\epsilon) $~\cite{syed2010reduction,ross2011reduction}. Given the error of a typical supervised learning is $ \mathcal{O}\left(T\epsilon\right) $, this demonstrates the additional cost due to covariate shift when solving an SPC task via standard supervised learning\footnote{The proofs regarding results $ \mathcal{O}(T^2\epsilon)$ and $ \mathcal{O}(T\epsilon)$ can be found in Appendix~\ref{sec:app-proofs}.}.

\subsection{Existing Techniques}
\label{sec:dagger}
Several approaches have been proposed to solve SPC tasks using supervised learning while keeping the error growing linearly instead of quadratically with $ T $~\cite{syed2010reduction,ross2011reduction,daume2009search}. Essentially, these methods reduce an SPC task to online learning. By further leveraging interactions with experts and no-regret algorithms that have strong guarantees on convex loss functions~\cite{kakade2009generalization}, at each iteration, these methods train one or multiple policies using standard supervised learning and improve the trained policies as the iteration continues. 


To illustrate, we denote the best policy at the $ i $th iteration (trained using all observations from the previous $ i-1 $ iterations) as $ \pi_{i} $ and for any policy $ \pi \in \Pi $ we have its expected loss under the observation distribution induced by $ \pi_{i} $ as $ l_{i}\left(\pi\right) =  \mathbb{E}_{\phi \sim d_{\pi_{i}}, a^* \sim \pi^*(\phi)} \left[l_{i}\left(\phi,\pi, a^*\right)\right], l_{i}\in \left[0,l_{max}\right]$\footnote{In online learning, the surrogate loss $ l $ can be seen as chosen by some adversary which varies at each iteration.}. In addition, we denote the minimal loss in hindsight after $ N \ge i $ iterations as $ \epsilon_{min} = \min_{\pi \in \Pi} \frac{1}{N}\sum_{i=1}^{N}l_{i}(\pi) $ (i.e., the training loss after using all observations from $ N $ iterations). Then, we can represent the average regret of this online learning program as $ \epsilon_{regret} = \frac{1}{N}\sum_{i=1}^{N}l_{i}(\pi_{i}) - \epsilon_{min}$. Using DAGGER~\cite{ross2011reduction} as an example method, the accumulated error difference becomes the summation of three terms: 

\vspace{-1em}

\begin{equation}
J\left(\hat{\pi}\right) \le T\epsilon_{min} + T\epsilon_{regret} +  \mathcal{O}(\frac{f\left(T, l_{max}\right)}{N}),
\label{eq:dagger}
\end{equation}

\noindent where $ f\left(\cdot\right) $ is the function of fixed $ T $ and $ l_{max} $. As $ N \rightarrow \infty $, the third term tends to $ 0 $ so as the second term if a no-regret algorithm such as the Follow-the-Leader~\cite{hazan2007logarithmic} is used.



The aforementioned approach provides a practical way to solve SPC tasks. However, it may require many iterations for obtaining a good policy. In addition, usually human experts or pre-defined controllers are needed for labeling the generated training data, which could be inefficient or difficult to generalize. For autonomous driving, we want the iteration number to be minimal since it directly corresponds to the number of accidents. This requires the generation of training data being accurate, efficient, and sufficient.




\section{ADAPS}
\label{sec:adaps}
In the following, we present theoretical analysis of our framework and introduce our framework pipeline. 

\subsection{Theoretical Analysis}
We have evaluated our approach against existing learning mechanisms such as DAGGER~\cite{ross2011reduction}, with our method's results proving to be more effective. Specifically, DAGGER~\cite{ross2011reduction} assumes that an underlying learning algorithm has access to a \textit{reset model}. So, the training examples can be obtained only \textit{online} by putting an agent to its initial state distribution and executing a learned policy, thus achieving  ``small changes'' at each iteration~\cite{ross2011reduction,daume2009search,kakade2002approximately,bagnell2004policy}. In comparison, our method allows a learning algorithm to access a \textit{generative model} so that the training examples can be acquired \textit{offline} by putting an agent to arbitrary states during the analysis of an accident and letting a \textit{generative} model simulate its behavior.  This approach results in massive training data, thus achieving ``large changes'' of a policy at one iteration. 

Additionally, existing techniques such as DAGGER~\cite{ross2011reduction} usually incorporate the demonstrations of a few experts into training. Because of the \textit{reset} model assumption and the lack of a diversity requirement on experts, these demonstrations can be homogeneous. In contrast, using our parameterized model to retrace and analyze each accident, the number of recovery actions obtained can be multiple orders of magnitude higher. Subsequently, we can treat the generated trajectories and the additional data generated based on them (described in Section~\ref{sec:data_detect}) as running a learned policy to sample independent expert trajectories at different states, since 1) a policy that is learned using DNNs can achieve a small training error and 2) our model provides near-exhaustive coverage of the configuration space of a vehicle. With these assumptions, we derive the following theorem.

\begin{theorem}
	If the surrogate loss $ l $ upper bounds the true cost $ C $, by collecting $ K $ trajectories using ADAPS at each iteration, with probability at least $ 1 - \mu $, $ \mu \in (0,1) $, we have the following guarantee:
	\vspace{-1.0em}
	\begin{equation*}
	J\left(\hat{\pi}\right) \le J\left(\bar{\pi}\right) \le T\hat{\epsilon}_{min} + T\hat{\epsilon}_{regret} + \mathcal{O}\left(Tl_{max}\sqrt{\frac{\log{\frac{1}{\mu}}}{KN}}\right).
	\end{equation*}
	\label{thm:adaps}
\end{theorem}

\vspace{-.5em}

\begin{proof}
	See Appendix~\ref{sec:app-adaps-proof}.
\end{proof}
  
Theorem~\ref{thm:adaps} provides a bound for the expected cost-to-go of the best learned policy $ \hat{\pi} $ based on the empirical error of the best policy in $ \Pi $ (i.e., $ \hat{\epsilon}_{min} $) and the empirical average regret of the learner (i.e., $ \hat{\epsilon}_{regret} $). The second term can be eliminated if a no-regret algorithm such as Follow-the-Leader~\cite{hazan2007logarithmic} is used and the third term suggests that we need the number of training examples $ KN $ to be $\mathcal{O}\left(T^2l^2_{max}\log{\frac{1}{\mu}}\right) $ in order to have a negligible generalization error, which is easily achievable using ADAPS. Summarizing these changes, we derive the following Corollary.

\begin{corollary}
	If $ l $ is convex in $ \pi $ for any $ s $ and it upper bounds $ C $, and Follow-the-Leader is used to select the learned policy, then for any $ \epsilon > 0 $, after collecting $ \mathcal{O}\left(\frac{T^2l^2_{max}\log{\frac{1}{\mu}}}{\epsilon^2}\right) $ training examples, with probability at least $ 1 - \mu $, $ \mu \in (0,1) $, we have the following guarantee:
	\begin{equation*}
	J\left(\hat{\pi}\right) \le J\left(\bar{\pi}\right) \le T\hat{\epsilon}_{min} + \mathcal{O}\left(\epsilon\right).
	\end{equation*}
	\label{col:adaps}
\end{corollary}

\vspace{-1.2em}

\begin{proof}
	Following Theorem~\ref{thm:adaps} and the aforementioned deduction.
\end{proof}

Now we only need the best policy to have a small training error $ \hat{\epsilon}_{min} $. This can be achieved using DNNs since they have rich representing capabilities.

\subsection{Framework Pipeline}


The pipeline of our framework is the following. First, in \textit{SimLearner}, we test a learned policy by letting it control an AV. During the testing, an accident may occur, in which case the trajectory of the vehicle and the full specifications of the situation (e.g., positions of obstacles, road configuration, etc.) are known. Next, we switch to  \textit{SimExpert} and replicate the specifications of the accident so that we can ``solve'' the accident (i.e., find alternative safe trajectories and dangerous zones). After obtaining the solutions, we then use them to generate additional training data in \textit{SimLearner}, which will be combined with previously generated data to update the policy. Finally, we test the updated policy again.


%
%

\section{Policy Learning}
\label{sec:policy}
In this section, we will detail our control policy by first explaining our design rationale then formulating our problem and introducing the training data collection.

Driving is a hierarchical decision process. In its simplest form, a driver needs to constantly monitor the road condition, decide it is ``safe'' or ``dangerous'', and make corresponding maneuvers. When designing a control policy for AVs, we need to consider this hierarchical aspect. In addition, driving is a temporal behavior. Drivers need reaction time to respond to various road situations~\cite{johansson1971drivers,mcgehee2000driver}. A Markovian-based control policy will not model this aspect and instead likely to give a vehicle jerky motions. Consider these factors, we propose a \textit{hierarchical} and \textit{memory-enabled} control policy. 

The task we consider is autonomous driving via a single front-facing camera. Our control policy consists of three modules: \textit{Detection}, \textit{Following}, and \textit{Avoidance}. The \textit{Detection} module keeps monitoring road conditions and activates either \textit{Following} or \textit{Avoidance} to produce a steering command. All these modules are trained via end-to-end imitation learning and share a similar network specification which is detailed in Appendix~\ref{sec:network-specs}.


\vspace{-0.2em}

\subsection{End-to-end Imitation Learning}
\label{subsec:end2end}

The objective of imitation learning is to train a model that behaves or makes decisions like an expert through demonstrations. The model could be a classifier or a regresser $ \pi $ parameterized by $ \mathbf{\theta}_{\pi} $:
\vspace*{-1em}

\begin{equation}
\hat{\theta} = \argmin_{\theta_{\pi}}\sum_{t=1}^{T} \mathcal{F}\left(\pi\left(\phi_{t}; \mathbf{\theta_{\pi}}\right),a^*_{t} \right),
\end{equation}

\noindent where $ \mathcal{F} $ is a distance function.

The end-to-end aspect denotes the mapping from raw observations to decision/control commands. For our policy, we need one decision module $ \pi_{Detection} $ and two control modules $ \pi_{Following} $ and $ \pi_{Avoidance} $. The input for $ \pi_{Detection} $ is a sequence of annotated images while the outputs are binary labels indicating whether a road condition is dangerous or safe. The inputs for $ \pi_{Following} $ and $ \pi_{Avoidance} $ are sequences of annotated images while the outputs are steering angles. Together, these learned policies form a hierarchical control mechanism enabling an AV to drive safely on roads and avoid obstacles when needed.


\subsection{Training Data Collection}
\label{subsec:data}


For training \textit{Following}, inspired by the technique used by Bojarski et al.~\cite{Bojarski2016}, we collect images from three front-facing cameras behind the main windshield: one at the center, one at the left side, and one at the right side. The image from the center camera is labeled with the exact steering angle while the images from the other two cameras are labeled with adjusted steering angles. However, once \textit{Following} is learned, it only needs images from the center camera to operate.


For training \textit{Avoidance}, we rely on \textit{SimExpert}, which can generate numerous intermediate collision-free trajectories between the first moment and the last moment of a potential accident (see Section~\ref{subsec:solving}). By positioning an AV on these trajectories, we collect images from the center front-facing camera along with corresponding steering angles. The training of \textit{Detection} requires a more sophisticated mechanism and is the subject of the next section.

\section{Learning from Accidents}
\label{sec:accident}

We explain how we analyze an accident in \textit{SimExpert} and use the generated data to train the \textit{Avoidance} and \textit{Detection} modules of our policy. \textit{SimExpert} is built based on the multi-agent simulator WarpDriver~\cite{wolinski2016warpdriver}.

%

\subsection{Solving Accidents}
\label{subsec:solving}

When an accident occurs, we know the trajectory of the tested vehicle for the latest $K$ frames, which we note as a collection of states $\mathcal{S} = \bigcup_{k \in [\![1, K]\!]} \mathbf{s}_k$, where each state $\mathbf{s}_k \in \mathbb{R}^4$ contains the 2-dimensional position and velocity vectors of the vehicle.
Then, there are three notable states on this trajectory that we need to track.
The first is the earliest state where the vehicle involved in an accident (is in a collision) $\mathbf{s}_{k_a}$ (at frame $k_a$).
The second is the last state $\mathbf{s}_{k_l}$ (at frame $k_l$) where the expert algorithm can still avoid a collision.
The final one is the first state $\mathbf{s}_{k_f}$ (at frame $k_f$) where the expert algorithm perceives the interaction leading to the accident with the other involved agent, before that accident.

In order to compute these notable states, we briefly recall the high-level components of WarpDriver~\cite{wolinski2016warpdriver}. This collision-avoidance algorithm consists of two parts.
The first is the function $p$, which given the current state of an agent $\mathbf{s}_k$ and any prediction point $\mathbf{x} \in \mathbb{R}^3$ in 2-dimensional space and time (in this agent's referential), gives the probability of that agent's colliding with any neighbor $p(\mathbf{s}_k, \mathbf{x}) \in [0, \: 1]$.
The second part is the solver, which based on this function, computes the agent's probability of colliding with neighbors along its future trajectory starting from a state $\mathbf{s}_k$ (i.e., computed for $\mathbf{x}$ spanning the future predicted trajectory of the agent, we denote this probability $P\left(\mathbf{s}_k\right)$), and then proposes a new velocity to lower this probability.
Subsequently, we can initialize an agent in this algorithm to any state $\mathbf{s}_k \in \mathcal{S}$ and compute a new trajectory consisting of $\hat{K}$ new states $\mathcal{\hat{S}}_k = \bigcup_{\hat{k} \in [\![1, \hat{K}]\!]} \mathbf{\hat{s}}_{\hat{k}}$, where $\mathbf{\hat{s}}_1 = \mathbf{s}_k$.

Additionally, since $\mathbf{x} = \vecth{0}{0}{0}$ in space and time in an agent's referential represents the agent's position at the current time (we can use this point $\mathbf{x}$ with function $p$ to determine if the agent is currently colliding with anyone), we find $\mathbf{s}_{k_a}$ where $ k_a = \min(k)$ subject to $k \in [\![1, K]\!]$ and $p(\mathbf{s}_{k}, \vecth{0}{0}{0}) > 0$. We note that a trajectory $\mathcal{\hat{S}}_k$ produced by the expert algorithm could contain collisions (accounting for vehicle dynamics) depending on the state $\mathbf{s}_k$ that it was initialized from.
We can denote the set of colliding states along this trajectory as $coll(\mathcal{\hat{S}}_k) = \{ \mathbf{\hat{s}}_{\hat{k}} \in \mathcal{\hat{S}}_k \: | \: p(\mathbf{\hat{s}}_{\hat{k}}, \vecth{0}{0}{0}) > 0 \}$.
Then, we can compute $\mathbf{s}_{k_l}$ where $k_l = max(k)$ subject to $k \in [\![1, k_a]\!]$ and $coll(\mathcal{\hat{S}}_k) = \emptyset$. Finally, we can compute $\mathbf{s}_{k_f}$ with $k_f = 1 + max(k)$ subject to $k \in [\![1, k_l]\!]$ and $P(\mathbf{s}_k) = 0$.

Knowing these notable states, we can solve the accident situation by computing the set of collision-free trajectories $solve(\mathcal{S}) = \{ \mathcal{\hat{S}}_k \: | \: k \in [\![k_f, k_l]\!] \}$. An example can be found in Appendix~\ref{sec:app-traj}. These trajectories can then be used to generate training examples in \textit{SimLearner} in order to train the \textit{Avoidance} module.

\subsection{Additional Data Coverage}
\label{sec:data_detect}

\begin{figure}
\begin{center}
\includegraphics[width=1\columnwidth]{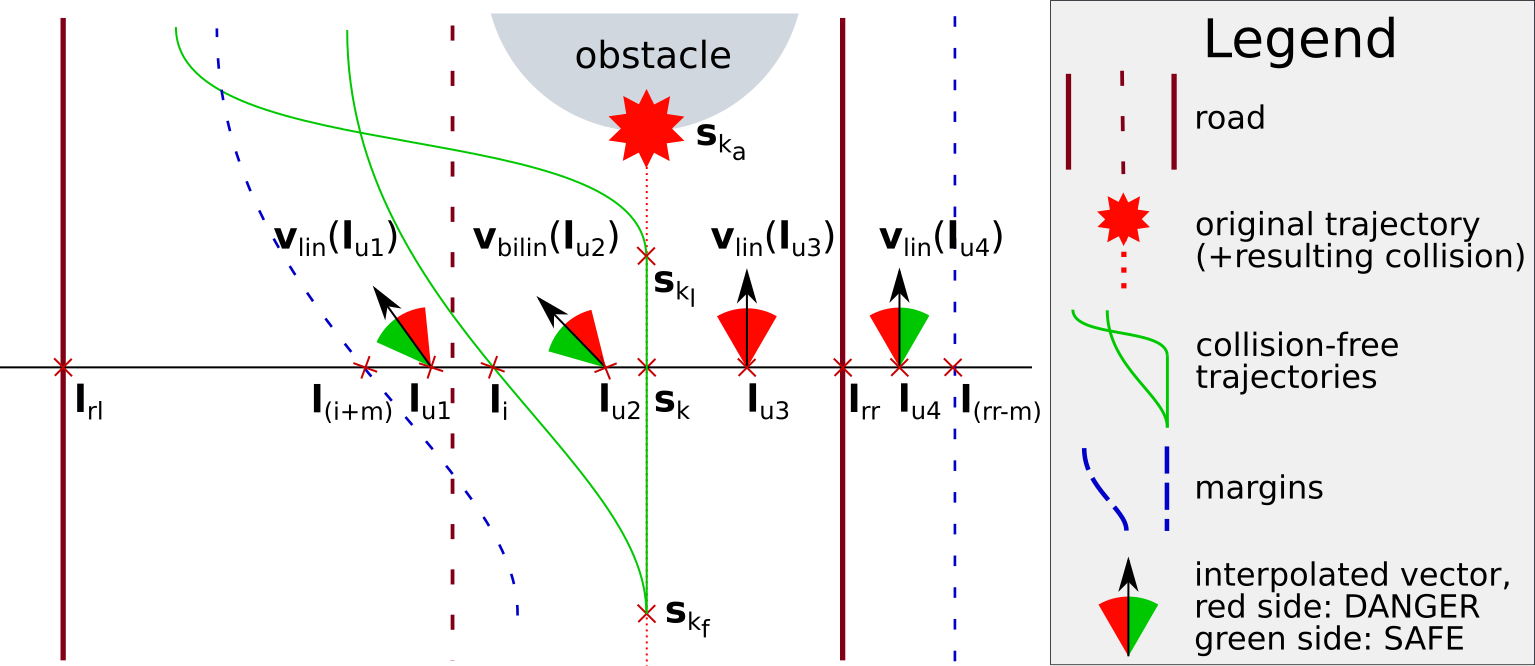}
\end{center}
\vspace{-1.3em}
\caption{
Illustration of important points and DANGER/SAFE labels from Section~\ref{sec:accident} for a vehicle traveling on the right lane of a straight road, with an obstacle in front.
Labels are shown for four points $\{ \mathbf{l}_{u1}, \mathbf{l}_{u2}, \mathbf{l}_{u3}, \mathbf{l}_{u4} \}$ illustrating the four possible cases.
}
\label{fig:points}
\vspace{-2em}
\end{figure}

The previous step generated collision-free trajectories $solve(\mathcal{S})$ between $\mathbf{s}_{k_f}$ and $\mathbf{s}_{k_l}$.
It is possible to build on these trajectories if the tested steering algorithm has particular data/training requirements.
Here we detail the data we derive in order to train the \textit{Detection} module, where the task is to determine if a situation is dangerous and tell \textit{Avoidance} to address it.

To proceed, we essentially generate a number of trajectories parallel to $\{ \mathbf{s}_{k_f}, ..., \mathbf{s}_{k_a} \}$, and for each position on them, generate several images for various orientations of the vehicle.
These images are then labeled based on under-steering/over-steering as compared to the ``ideal'' trajectories in $solve(\mathcal{S})$.
This way, we scan the region of the road before the accident locus, generating several images (different vehicle orientations) for each point in that region.

In summary (a thorough version can be found in Appendix~\ref{sec:app-accidents}), and as depicted in Figure~\ref{fig:points}, at each state $\mathbf{s}_k$, we construct a line perpendicular to the original trajectory.
Then on this line, we define three points and a margin $g=0.5~m$.
The first point $\mathbf{l}_i$ is the furthest (from $\mathbf{s}_k$) intersection between this line and the collision-free trajectories $solve(\mathcal{S})$.
The other two points $\{ \mathbf{l}_{rl}, \mathbf{l}_{rr} \}$ are the intersections between the constructed line and the left and right road borders, respectively.
From these points, a generated image at a position $\mathbf{l}_u$ along the constructed line and with a given direction vector has either a DANGER or SAFE label (red and green ranges in Figure~\ref{fig:points}) depending on the direction vector being on the ``left'' or ``right'' of the vector resulting from the interpolation of the velocity vectors of states belonging to nearby collision-free trajectories (bilinear interpolation if $\mathbf{l}_u$ is between two collision-free trajectories, linear otherwise).

If a point is on the same side of the original trajectory as the collision-free trajectories ($\mathbf{l}_{u1}$ and $\mathbf{l}_{u2}$ in  Figure~\ref{fig:points}, $\mathbf{l}_{u1}$ is ``outside'' but within the margin $g$ of the collision-free trajectories, $\mathbf{l}_{u2}$ is ``inside'' the collision-free trajectories), the label is SAFE on the exterior of the avoidance maneuver, and DANGER otherwise.

If a point is on the other side of the original trajectory as compared to the collision-free trajectories ($\mathbf{l}_{u3}$ and $\mathbf{l}_{u4}$ in Figure~\ref{fig:points})), inside the road ($\mathbf{l}_{u3}$) the label is always DANGER, while outside but within the margin $g$ of the road ($\mathbf{l}_{u4}$), the label is DANGER when directed towards the road, and SAFE otherwise.


\section{Experiments}
\label{sec:exp}

%
%


We test our framework in three scenarios: a straight road representing a linear geometry, a curved road representing a non-linear geometry, and an open ground. The first two scenarios demonstrate on-road situations with a static obstacle while the last one demonstrates an off-road situation with a dynamic obstacle. The specifications of our experiments are detailed in Appendix~\ref{sec:app-exp}. 

For evaluation, we compare our policy to the ``flat policy'' that essentially consists of a single DNN ~\cite{chen2015deepdriving,Xu2017end,zhang2017query,codevilla2017end}. Usually, this type of policy contains a few convolutional layers followed by a few dense layers. Although the specifications may vary, without human intervention, they are mainly limited to single-lane following~\cite{codevilla2017end}. In this work, we select Bojarski et al.~\cite{Bojarski2016} as an example network, as it is one of the most tested control policies. In the following, we will first demonstrate the effectiveness of our policy and then qualitatively illustrate the efficiency of our framework.

\subsection{Control Policy}

\begin{figure*}[th]
	\centering
	\includegraphics[width=7in]{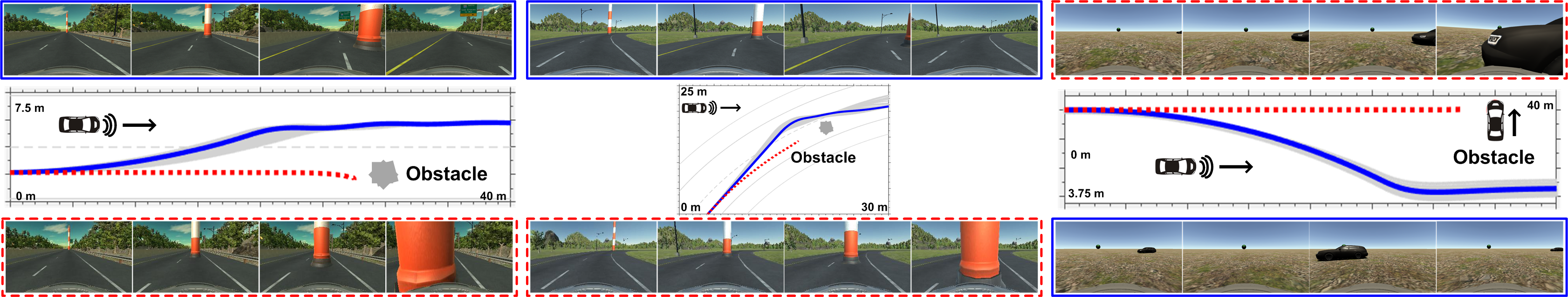}
	\vspace{-1em}
	\caption{LEFT and CENTER: the comparisons between our policy $ O_{full} $ (TOP) and Bojarski et al.~\cite{Bojarski2016}, $ B_{full} $ (BOTTOM). $ O_{full} $ can steer the AV away from the obstacle while $ B_{full} $ causes collision. RIGHT: the accident analysis results on the open ground. We show the accident caused by an adversary vehicle (TOP); then we show after additional training the AV can now avoid the adversary vehicle (BOTTOM).}
	\label{fig:examples}
	\vspace{-0.5em}
\end{figure*}

\subsubsection{On-road}

We derive our training datasets from \textit{straight road with or without an obstacle} and \textit{curved road with or without an obstacle}.
This separation allows us to train multiple policies and  test the effect of \textit{learning from accidents} using our policy compared to Bojarski et al.~\cite{Bojarski2016}. By progressively increasing the training datasets, we obtain six policies for evaluation:

\begin{itemize}
	\item Our own policy: trained with only lane-following data $O_{follow}$; $O_{follow}$ additionally trained after analyzing one accident on the straight road $O_{straight}$; and $O_{straight}$ additionally trained after producing one accident on the curved road $O_{full}$.
	\item Similarly, for the policy from Bojarski et al.~\cite{Bojarski2016}: $B_{follow}$, $B_{straight}$, and $B_{full}$.
\end{itemize}

We first evaluate $B_{follow}$ and $O_{follow}$ using both the straight and curved roads by counting how many laps (out of 50) the AV can finish. As a result, both policies managed to finish all laps while keeping the vehicle in the lane. We then test these two policies on the straight road with a static obstacle added. Both policies result in the vehicle collides into the obstacle, which is expected since no accident data were used during the training.

Having the occurred accident, we can now use \textit{SimExpert} to generate additional training data to obtain $B_{straight}$\footnote{The accident data are only used to perform a regression task as the policy by Bojarski et al.~\cite{Bojarski2016} does not have a classification module.} and $O_{straight}$.
As a result, $B_{straight}$ continues to cause collision while $O_{straight}$ avoids the obstacle. Nevertheless, when testing $O_{straight}$ on the curved road with an obstacle, accident still occurs because of the corresponding accident data are not yet included in training.

By further including the accident data from the curved road into training, we obtain $ B_{full} $ and $ O_{full}$. $ O_{full}$ manages to perform both lane-following and collision avoidance in all runs. $B_{full}$, on the other hand, leads the vehicle to drift away from the road. 

For the studies involved an obstacle, we uniformly sampled 50 obstacle positions on a $3~m$ line segment that is perpendicular to the direction of a road and in the same lane as the vehicle.
We compute the success rate as how many times a policy can avoid the obstacle (while stay in the lane) and resume lane-following afterwards. The results are shown in Table~\ref{tb:acc} and example trajectories are shown in Figure~\ref{fig:examples} LEFT and CENTER.



\begin{table*}[ht!]
	\centering
	\small
	\tabcolsep=0.1cm
	\scalebox{0.8}{
		\begin{tabular}{ccccccccc}
			\toprule
			& \multicolumn{5}{c}{Training Module (Data) } & \multicolumn{3}{c}{Other Specs}  \\        
			\cmidrule(l){2-6} \cmidrule(l){7-9}      
			Scenarios  & \textit{Following (\#Images)} & \textit{Avoidance (\#Images)}  & \textit{Detection (\#Images)} & Total & Data Augmentation & \#Safe Trajectories & Road Type & Obstacle   
			\\
			\midrule
			Straight road & 33~642 & 34~516  & 32~538  & 97~854 & $ 212 $x & 74  &  on-road   & static 
			\\
			\midrule
			Curved road   & 31~419 & 33~624  & 71~859  & 136~855 & $ 98 $x & 40  &  on-road   & static 
			\\
			\midrule
			Open ground   & 30~000 & 33~741  & 67~102 & 130~843 & $ 178 $x & 46  &  off-road  &  dynamic
			\\
			\bottomrule
		\end{tabular}}
	\caption{\textbf{Training Data Summary:} Our method can achieve over \textbf{200} times more training examples than DAGGER~\cite{ross2011reduction} at one iteration leading to large improvements of a policy.} 
	
	\vspace{-2em}
	\label{tb:train-data}
\end{table*}

\subsubsection{Off-road}

We further test our method on an open ground which involves a dynamic obstacle. The AV is trained heading towards a green sphere while an adversary vehicle is scripted to collide with the AV on its default course. The result showing our policy can steer the AV away from the adversary vehicle and resume its direction to the sphere target. This can be seen in Figure~\ref{fig:examples} RIGHT.

\begin{table}[ht!]
\centering
\small
\tabcolsep=0.1cm
\scalebox{0.8}{
	\begin{tabular}{ccccccc}
		\toprule
		& \multicolumn{6}{c}{Test Policy and Success Rate (out of 50 runs)}  \\        
		\cmidrule(l){2-7}       
		Scenario & $ B_{follow} $ & $ O_{follow} $  & $ B_{straight} $ & $ O_{straight} $ & $ B_{full} $ & $ O_{full} $   \\
		\midrule
		Straight rd. / Curved rd. & 100\% & 100\%  & 100\%  & 100\%  & 100\%  & 100\% 
		\\
		\midrule
		Straight rd. + Obst. & 0\% & 0\%  & \textbf{0\%}  & \textbf{100\%}  & \textbf{0\%}  & \textbf{100\%}
		\\
		\midrule
		Curved rd. + Obst. & 0\% & 0\%  & 0\%  & 0\%  &\textbf{ 0\% } & \textbf{100\%}
		\\
		\bottomrule
	\end{tabular}}
	\caption{\textbf{Test Results of On-Road Scenarios:} Our policies $ O_{straight} $ \& $ O_{full} $ can lead to robust collision avoidance and lane-following behaviors. }
	\label{tb:acc}
	\vspace{-2em}
\end{table}

\subsection{Algorithm Efficiency}

The key to rapid policy improvement is to generate training data accurately, efficiently, and sufficiently. Using \textit{principled simulations} covers the first two criteria, now we demonstrate the third. Compared to the average number of training data collected by DAGGER~\cite{ross2011reduction} at one iteration, our method can achieve over 200 times more training examples for one iteration\footnote{The result is computed via dividing the total number of training images via our method by the average number of training data collected using the safe trajectories in each scenario.}. This is shown in Table~\ref{tb:train-data}.

In Figure~\ref{fig:straight-tsne}, we show the visualization results of images collected using our method and DAGGER~\cite{ross2011reduction} within one iteration via progressively increasing the number of sampled trajectories. Our method generates much more heterogeneous training data, which when produced in a large quantity can greatly facilitate the update of a control policy.

\begin{figure}[th]
	\centering
	\includegraphics[width=0.7\columnwidth]{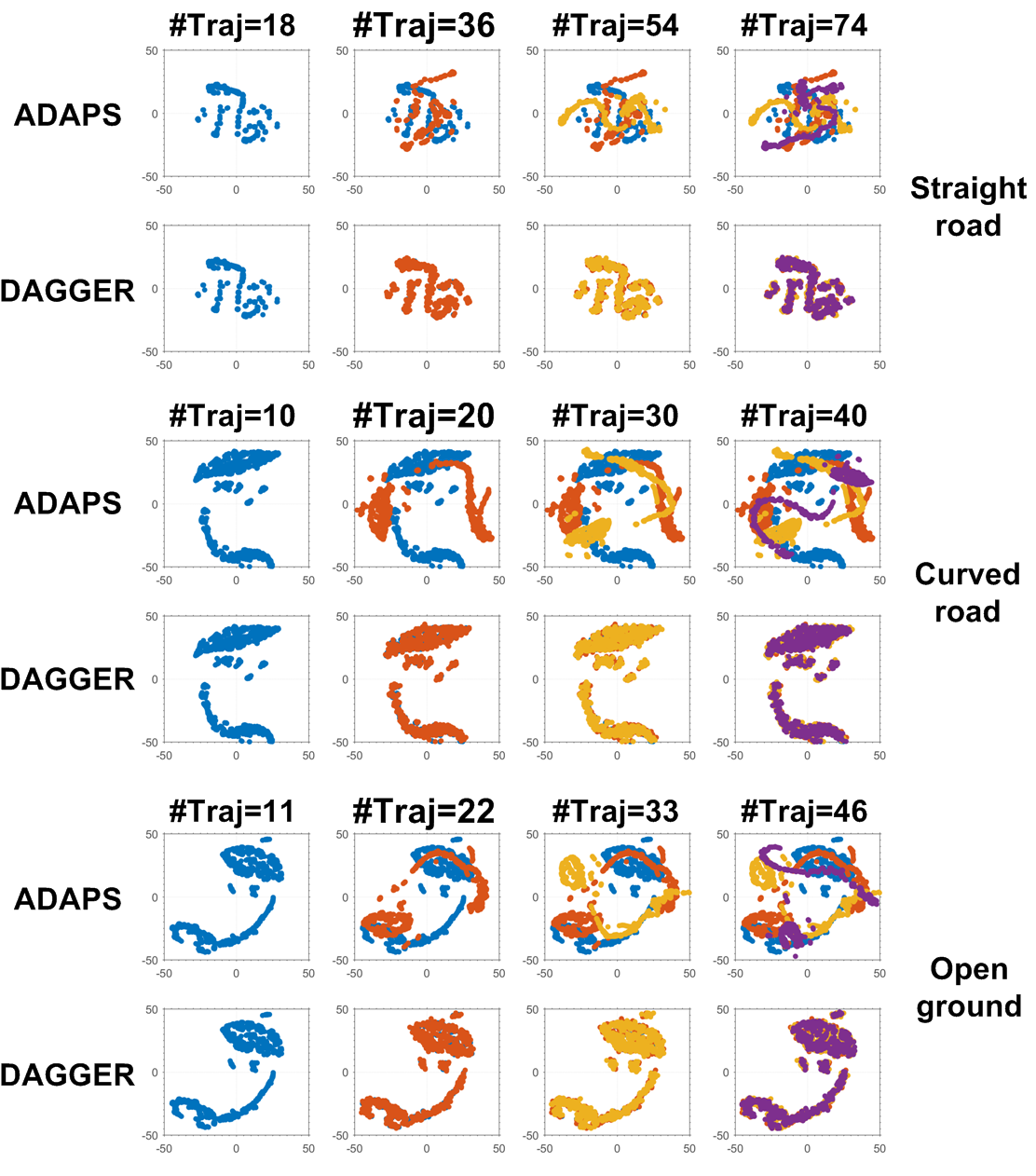}
	\vspace{-1em}
	\caption{The visualization results of collected images using t-SNE~\cite{maaten2008visualizing}. Our method can generate more heterogeneous training data compared to DAGGER~\cite{ross2011reduction} at one iteration as the sampled trajectories progress. }
	\label{fig:straight-tsne}
	\vspace{-1em}
\end{figure}

\section{Conclusion}
\label{sec:conclude}

In this work, we have proposed ADAPS, a framework that consists of two simulation platforms and a control policy. Using ADAPS, one can easily simulate accidents. Then, ADAPS will automatically retrace each accident, analyze it, and plan alternative safe trajectories. With the additional training data generation technique, our method can produce a large number of heterogeneous training examples compared to existing methods such as DAGGER~\cite{ross2011reduction}, thus representing a more efficient learning mechanism. Our hierarchical and memory-enabled policy offers robust collision avoidance behaviors that previous policies fail to achieve. We have evaluated our method using multiple simulated scenarios, in which our method shows a variety of benefits.

There are many future directions. First of all, we would like to combine long-range vision into ADAPS so that an AV can plan ahead in time. Secondly, the generation of accidents can be parameterized using knowledge from traffic engineering studies. Lastly, we would like to combine more sensors and fuse their inputs so that an AV can navigate in more complicated traffic scenarios~\cite{CityFlowRecon:2017}.

\section*{Acknowledgment}
The authors would like to thank US Army Research Office and UNC Arts \& Science Foundation, and Dr. Feng ``Bill'' Shi for insightful discussions.

\clearpage 
\bibliographystyle{IEEEtran}
\bibliography{root}

\clearpage 
\section{Appendix}

\subsection{Solving An SPC Task}
\label{sec:app-proofs}
%
%
%
%

We show the proofs of solving an SPC task using standard supervised learning, DAGGER~\cite{ross2011reduction}, and ADAPS, respectively. We use ``state'' and "observation'' interchangeably here as for these proofs we can always find a deterministic function to map the two.

\subsubsection{Supervised Learning}
The following proof is adapted and simplified from Ross et al.~\cite{ross2011reduction}. We include it here for completeness.

\begin{theorem}
	Consider a $ T $-step control task. Let $ \epsilon = \mathbb{E}_{\phi \sim d_{\pi^*}, a^* \sim \pi^*(\phi)} \left[l\left(\phi,\pi, a^*\right)\right] $ be the observed surrogate loss under the training distribution induced by the expert's policy $ \pi^* $. We assume $ C \in \left[0, C_{max}\right] $ and $ l $ upper bounds the 0-1 loss. $ J\learn $ and $ J\expert $ denote the cost-to-go over $ T $ steps of executing $ \pi $ and $ \pi^* $, respectively. Then, we have the following result:
	\begin{equation*}
		J\learn \le J\expert + C_{max}T^2\epsilon.
	\end{equation*}
	\label{thm:sl}
\end{theorem}

\begin{proof}
	In order to prove this theorem, we introduce the following notation and definitions:
	\begin{itemize}
		\item $ d_{t,c}^{\pi} $: the state distribution at $ t $ as a result of the following event: $ \pi $ is executed and has been choosing the same actions as $ \pi^* $ from time $ 1 $ to $ t-1 $.
		\item $ p_{t-1} \in \left[0,1\right] $: the probability that the above-mentioned event holds true.
		\item $ d_{t,e}^{\pi} $: the state distribution at $ t $ as a result of the following event: $ \pi $ is executed and has chosen at least one different action than $ \pi^* $ from time $ 1 $ to $ t-1 $.
		\item $ (1 - p_{t-1}) \in \left[0,1\right] $: the probability that the above-mentioned event holds true.
		\item $ d_{t}^{\pi} = p_{t-1}d_{t,c}^{\pi} + (1-p_{t-1})d_{t,e}^{\pi} $: the state distribution at $ t $.
		\item $ \epsilon_{t,c} $: the probability that $ \pi $ chooses a different action than $ \pi^* $ in $ d_{t,c}^{\pi} $.
		\item $ \epsilon_{t,e} $: the probability that $ \pi $ chooses a different action than $ \pi^* $ in $ d_{t,e}^{\pi} $.
		\item $ \epsilon_{t} = p_{t-1}\epsilon_{t,c} + (1-p_{t-1})\epsilon_{t,e} $: the probability that $ \pi $ chooses a different action than $ \pi^* $ in $ d_{t}^{\pi} $.
		\item $ C_{t,c} $: the expected immediate cost of executing $ \pi $ in $ d_{t,c}^{\pi} $.
		\item $ C_{t,e} $: the expected immediate cost of executing $ \pi $ in $ d_{t,e}^{\pi} $.
		\item $ C_{t} = p_{t-1}C_{t,c} + (1-p_{t-1})C_{t,e} $: the expected immediate cost of executing $ \pi $ in $ d_{t}^{\pi} $.
		\item $ C_{t,c}^{*} $: the expected immediate cost of executing $ \pi^{*} $ in $ d_{t,c}^{\pi} $.

		\item $ C_{max} $: the upper bound of an expected immediate cost.
		\item $ J\learn  = \sum_{t=1}^{T} C_{t}$: the cost-to-go of executing $ \pi $ for $ T $ steps.
		\item $ J\expert = \sum_{t=1}^{T} C_{t,c}^* $: the cost-to-go of executing $ \pi^* $ for $ T $ steps.
	\end{itemize}
    \noindent The probability that the learner chooses at least one different action than the expert in the first $ t $ steps is: 
	\begin{equation*}
		\left(1-p_{t}\right) = \left(1-p_{t-1}\right) + p_{t-1}\epsilon_{t,c}.
		\label{aux0}
	\end{equation*}
	\noindent This gives us $ \left(1-p_{t}\right) \le (1-p_{t-1}) + \epsilon_{t} $ since $ p_{t-1} \in \left[0,1\right]$. Solving this recurrence we arrive at:
	\begin{equation*}
		1 - p_{t} \le \sum_{i=1}^{t}\epsilon_{i}.
		\label{aux1}
	\end{equation*}
	\noindent Now consider in state distribution $ d_{t,c}^{\pi} $, if $ \pi $ chooses a different action than $ \pi^* $  with probability $ \epsilon_{t,c} $, then $ \pi $ will incur a cost at most $ C_{max} $ more than $ \pi^* $. This can be represented as:
	\begin{equation*}
		C_{t,c} \le C_{t,c}^{*} + \epsilon_{t,c}C_{max}.
		\label{aux2}
	\end{equation*}
	\noindent Thus, we have:
	\begin{equation*} 
		\begin{split}
		C_{t} & = p_{t-1}C_{t,c} + (1 - p_{t-1})C_{t,e} \\
		& \le p_{t-1}C_{t,c}^{*} + p_{t-1}\epsilon_{t,c}C_{max} + (1-p_{t-1})C_{max} \\
		& = p_{t-1}C_{t,c}^{*} + (1 - p_{t})C_{max} \\
		& \le C_{t,c}^{*} + (1 - p_{t})C_{max}\\
		& \le C_{t,c}^{*} + C_{max}\sum_{i=1}^{t}\epsilon_{i}.
		\end{split}
	\end{equation*}
	
	\noindent We sum the above result over $ T $ steps and use the fact $ \frac{1}{T}\sum_{t=1}^{T}\epsilon_{t} \le \epsilon $: 
	
	\begin{equation*} 
	\begin{split}
	J\learn & \le J\expert + C_{max}\sum_{t=1}^{T}\sum_{i=1}^{t}\epsilon_{i} \\
	& = J\expert + C_{max}\sum_{t=1}^{T}(T+1-t)\epsilon_{t}\\
	& \le J\expert + C_{max}T\sum_{t=1}^{T}\epsilon_{t}\\
	& \le J\expert + C_{max}T^2\epsilon.
	\end{split}
	\end{equation*}

\end{proof}


\subsubsection{DAGGER}

The following proof is adapted from Ross et al.~\cite{ross2011reduction}. We include it here for completeness. Note that for Theorem~\ref{thm:dagger}, we have arrived at the different third term as of Ross et al.~\cite{ross2011reduction}. 

\begin{lemma}
	\cite{ross2011reduction} Let $ P $ and $ Q $ be any two distributions over elements $ x \in \mathcal{X} $ and $ f: \mathcal{X} \rightarrow \mathbb{R} $, any bounded function such that $ f(x) \in \left[a,b\right] $ for all $ x \in \mathcal{X} $. Let the range $ r = b-a $. Then $ |\mathbb{E}_{x \sim P}\left[f(x)\right] - \mathbb{E}_{x \sim Q}\left[f(x)\right] | \le \frac{r}{2}\lVert P-Q \rVert_{1} $.
	\label{lemma-1}
\end{lemma}

\begin{proof}
	\begin{equation*} 
	\begin{split}
	& |\mathbb{E}_{x \sim P}\left[f(x)\right] - \mathbb{E}_{x \sim Q}\left[f(x)\right] | \\
	& = | \int_{x}P(x)f(x)dx - \int_{x}Q(x)f(x)dx |\\
	& = | \int_{x}f(x)\left(P(x)-Q(x)\right)dx |\\
	& = | \int_{x}\left(f(x)-c\right)\left(P(x)-Q(x)\right)dx |, \forall c \in \mathbb{R}\\
	& \le \int_{x}|f(x)-c||P(x)-Q(x)|dx \\
	& \le \max_{x}|f(x)-c|\int_{x}|P(x)-Q(x)|dx \\
	& = \max_{x}|f(x)-c|\lVert P-Q \rVert_{1}. 
	\end{split}
	\end{equation*}
	Taking $ c = a + \frac{r}{2} $ leads to $ \max_{x}|f(x)-c| \le \frac{r}{2} $ and proves the lemma.
\end{proof}

\begin{lemma}
	\cite{ross2011reduction} Let $ \hat{\pi}_{i} $ be the learned policy, $ \pi^* $ be the expert's policy, and $ \pi_{i} $ be the policy used to collect training data with probability $ \beta_{i} $ executing $ \pi^* $ and probability $  1 - \beta_{i}$ executing $ \hat{\pi}_{i} $ over $ T $ steps. Then, we have $\lVert d_{\pi_{i}} - d_{\hat{\pi}_{i}} \rVert_{1} \le 2\min(1, T\beta_{i})$.
\label{lemma-2}
\end{lemma}

\begin{proof}
	In contrast to $ d_{\hat{\pi}_{i}} $ which is the state distribution as the result of solely executing $ \hat{\pi}_{i} $, we denote $ d $ as the state distribution as the result of $ \pi_{i} $ executing $ \pi^* $ at least once over $ T $ steps. This gives us $ d_{\pi_{i}} = (1-\beta_{i})^Td_{\hat{\pi}_{i}} + \left(1-(1-\beta_{i})^T\right)d $. We also have the facts that for any two distributions $ P $ and $ Q $, $\lVert P-Q \rVert_{1} \le 2$ and $(1-\beta)^T \ge 1-\beta T, \forall \beta \in \left[0,1\right] $. Then, we have $ \lVert d_{\pi_{i}} - d_{\hat{\pi}_{i}} \rVert_{1} \le 2 $ and can further show:
	\begin{equation*} 
	\begin{split}
	\lVert d_{\pi_{i}} - d_{\hat{\pi}_{i}} \rVert_{1} & = \left(1-(1-\beta_{i})^T\right) \lVert d - d_{\hat{\pi}_{i}} \rVert_{1}\\
	& \le 2\left(1-(1-\beta_{i})^T\right) \\
	& \le 2T\beta_{i}.
	\end{split}
	\end{equation*}
\end{proof}

\begin{theorem}
	\cite{ross2011reduction} If the surrogate loss $ l \in \left[0,l_{max}\right] $ is the same as the cost function $ C $ or upper bounds it, then after $ N $ iterations of DAGGER:
	\begin{equation}
	J\left(\hat{\pi}\right) \le J\left(\bar{\pi}\right) \le T\epsilon_{min} + T\epsilon_{regret} + \mathcal{O}(\frac{f(T,l_{max})}{N}).
	\end{equation}
	\label{thm:dagger}
\end{theorem}

\begin{proof}
	Let $ l_{i}\left(\pi\right) =  \mathbb{E}_{\phi \sim d_{\pi_{i}}, a^* \sim \pi^*(\phi)} \left[l\left(\phi,\pi, a^*\right)\right]]$ be the expected loss of any policy $ \pi \in \Pi $ under the state distribution induced by the learned policy $ \pi_{i} $ at the $ i $th iteration and $ \epsilon_{min} = \min_{\pi \in \Pi} \frac{1}{N}\sum_{i=1}^{N}l_{i}(\pi) $ be the minimal loss in hindsight after $ N \ge i $ iterations. Then, $ \epsilon_{regret} = \frac{1}{N}\sum_{i=1}^{N}l_{i}(\pi_{i}) - \epsilon_{min}$ is the average regret of this online learning program. In addition, we denote the expected loss of any policy  $ \pi \in \Pi $ under its own induced state distribution as $ L\left(\pi\right) =  \mathbb{E}_{\phi \sim d_{\pi}, a^* \sim \pi^*(\phi)} \left[l\left(\phi,\pi, a^*\right)\right]]$ and consider $ \bar{\pi} $ as the mixed policy that samples the policies $ \{\hat{\pi}_{i}\}_{i=1}^N $ uniformly at the beginning of each trajectory. Using Lemma~\ref{lemma-1} and Lemma~\ref{lemma-2}, we can show:
	
	\begin{equation*} 
	\begin{split}
	L(\hat{\pi}_{i}) & = \mathbb{E}_{\phi \sim d_{\hat{\pi}_{i}}, a^* \sim \pi^*(\phi)} \left[l\left(\phi,\hat{\pi}_{i}, a^*\right)\right]\\
	& \le \mathbb{E}_{\phi \sim d_{\pi_{i}}, a^* \sim \pi^*(\phi)} \left[l\left(\phi,\hat{\pi}_{i}, a^*\right)\right] + \frac{l_{max}}{2}\lVert d_{\pi_{i}} - d_{\hat{\pi}_{i}} \rVert_{1} \\
	& \le \mathbb{E}_{\phi \sim d_{\pi_{i}}, a^* \sim \pi^*(\phi)} \left[l\left(\phi,\hat{\pi}_{i}, a^*\right)\right] + l_{max}\min\left(1,T\beta_{i}\right) \\
	& = l_{i}\left(\hat{\pi}_{i}\right) + l_{max}\min\left(1,T\beta_{i}\right).
	\end{split}
	\end{equation*}
	
	\noindent By further assuming $ \beta_{i} $ is monotonically decreasing and $ n_{\beta} = \argmax_{n} (\beta_{n} > \frac{1}{T}), n \le N $, we have the following:
	
	\begin{equation*} 
	\begin{split}
	\min_{i \in 1:N} L(\hat{\pi}_{i}) & \le L(\bar{\pi})\\
	& = \frac{1}{N} \sum_{i=1}^{N}L(\hat{\pi}_{i}) \\
	& \le \frac{1}{N} \sum_{i=1}^{N}l_{i}(\hat{\pi}_{i}) + \frac{l_{max}}{N} \sum_{i=1}^{N}\min\left(1,T\beta_{i}\right)   \\
	& = \epsilon_{min} + \epsilon_{regret} + \frac{l_{max}}{N} \left[n_{\beta} + T\sum_{i=n_{\beta}+1}^{N}\beta_{i}\right].
	\end{split}
	\end{equation*}
	
	\noindent Summing over $ T $ gives us:
	
	\begin{equation*} 
	J(\bar{\pi}) \le T\epsilon_{min} + T\epsilon_{regret} + \frac{Tl_{max}}{N} \left[n_{\beta} + T\sum_{i=n_{\beta}+1}^{N}\beta_{i}\right].
	\end{equation*}
	
	\noindent Define $ \beta_{i} = (1-\alpha)^{i-1} $, in order to have $ \beta_{i} \le \frac{1}{T} $, we need $ (1-\alpha)^{i-1} \le \frac{1}{T} $ which gives us $ i \le 1 + \frac{\log{\frac{1}{T}}}{\log{(1-\alpha)}} $. In addition, note now $ i = n_{\beta}$ and $\sum_{i=n_{\beta}+1}^{N}\beta_{i} = \frac{(1-\alpha)^{n\beta} - (1-\alpha)^N}{\alpha} \le \frac{1}{T\alpha}$, continuing the above derivation, we have:
	
	\begin{equation*}
		J\left(\bar{\pi}\right) \le T\epsilon_{min} + T\epsilon_{regret} + \frac{Tl_{max}}{N}\left(1 + \frac{\log{\frac{1}{T}}}{\log{(1-\alpha)}} + \frac{1}{\alpha}\right).
	\end{equation*}
	
	\noindent Given the fact $J\left(\hat{\pi}\right) = \min_{i \in 1:N} J(\hat{\pi}_{i}) \le J(\bar{\pi}) $ and representing the third term as $ \mathcal{O}(\frac{f(T,l_{max})}{N}) $, we have proved the theorem.
	
\end{proof}

\subsubsection{ADAPS}
\label{sec:app-adaps-proof}

With the assumption that we can treat the generated trajectories from our model and the additional data generated based on them as running a learned policy to sample independent expert trajectories at different states while performing policy roll-out, we have the following guarantee of ADAPS. To better understand the following theorem and proof, we recommend interested readers to read the proofs of Theorem~\ref{thm:sl} and~\ref{thm:dagger} first. 

\begin{theorem}
	If the surrogate loss $ l $ upper bounds the true cost $ C $, by collecting $ K $ trajectories using ADAPS at each iteration, with probability at least $ 1 - \mu $, $ \mu \in (0,1) $, we have the following guarantee:
	\begin{equation*}
	J\left(\hat{\pi}\right) \le J\left(\bar{\pi}\right) \le T\hat{\epsilon}_{min} + T\hat{\epsilon}_{regret} + \mathcal{O}\left(Tl_{max}\sqrt{\frac{\log{\frac{1}{\mu}}}{KN}}\right).
	\end{equation*}
\end{theorem}

\begin{proof}
	 Assuming at the $ i $th iteration, our model generates $ K $ trajectories. These trajectories are independent from each other since they are generated using different parameters and at different states during the analysis of an accident. For the $ k $th trajectory, $ k \in [\![1, K]\!] $, we can construct an estimate $ \hat{l}_{ik}(\hat{\pi}_{i}) = \frac{1}{T}\sum_{t=1}^{T}l_{i}\left(\phi_{ikt},\hat{\pi}_{i}, a^{*}_{ikt}\right) $, where $ \hat{\pi}_{i} $ is the  learned policy from data gathered in previous $ i-1 $ iterations. Then, the approximated expected loss $ \hat{l}_{i} $ is the average of these $ K $ estimates: $ \hat{l}_{i}(\hat{\pi}_{i}) = \frac{1}{K}\sum_{k=1}^{K}\hat{l}_{ik}(\hat{\pi}_{i}) $. We denote $ \hat{\epsilon}_{min} = \min_{\pi \in \Pi} \frac{1}{N}\sum_{i=1}^{N}\hat{l}_{i}(\pi) $ as the approximated minimal loss in hindsight after $ N $ iterations, then $ \hat{\epsilon}_{regret} = \frac{1}{N}\sum_{i=1}^{N}\hat{l}_{i}(\hat{\pi}_{i}) - \hat{\epsilon}_{min}$ is the approximated average regret. 
	 
	 Let $ Y_{i,k} = l_{i}(\hat{\pi}_{i}) - \hat{l}_{ik}(\hat{\pi}_{i}) $ and define random variables $ X_{nK+m} = \sum_{i=1}^{n}\sum_{k=1}^{K}Y_{i,k} + \sum_{k=1}^{m}Y_{n+1,k} $, for $ n \in [\![0, N-1]\!] $ and $ m \in [\![1, K]\!] $. Consequently, $ \{X_{i}\}_{i=1}^{NK} $ form a martingale and $ |X_{i+1} - X_{i}| \le l_{max} $. By Azuma-Hoeffding's inequality, with probability at least $ 1 - \mu $, we have $ \frac{1}{KN}X_{KN} \le l_{max}\sqrt{\frac{2\log{\frac{1}{\mu}}}{KN}} $. 
	 
	 Next, we denote the expected loss of any policy  $ \pi \in \Pi $ under its own induced state distribution as $ L\left(\pi\right) =  \mathbb{E}_{\phi \sim d_{\pi}, a^* \sim \pi^*(\phi)} \left[l\left(\phi,\pi, a^*\right)\right]]$ and consider $ \bar{\pi} $ as the mixed policy that samples the policies $ \{\hat{\pi}_{i}\}_{i=1}^N $ uniformly at the beginning of each trajectory. At each iteration, during the data collection, we only execute the learned policy instead of mix it with the expert's policy, which leads to $ L(\hat{\pi}_{i}) = l(\hat{\pi}_{i})$. Finally, we can show:
	 
 	\begin{equation*} 
 	\begin{split}
 	\min_{i \in 1:N} L(\hat{\pi}_{i}) & \le L(\bar{\pi})\\
 	& = \frac{1}{N} \sum_{i=1}^{N}L(\hat{\pi}_{i}) \\
 	& = \frac{1}{N} \sum_{i=1}^{N}l_{i}(\hat{\pi}_{i}) \\
 	& = \frac{1}{KN} \sum_{i=1}^{N}\sum_{k=1}^{K}\left(\hat{l}_{ik}(\hat{\pi}_{i}) +  Y_{i,k}\right)    \\
 	& = \frac{1}{KN} \sum_{i=1}^{N}\sum_{k=1}^{K}\hat{l}_{ik}(\hat{\pi}_{i}) +  \frac{1}{KN}X_{KN}    \\
 	& = \frac{1}{N} \sum_{i=1}^{N}\hat{l}(\hat{\pi}_{i}) + \frac{1}{KN}X_{KN}\\
 	& \le \frac{1}{N} \sum_{i=1}^{N}\hat{l}(\hat{\pi}_{i}) +  l_{max}\sqrt{\frac{2\log{\frac{1}{\mu}}}{KN}}  \\
 	& = \hat{\epsilon}_{min} + \hat{\epsilon}_{regret} + l_{max}\sqrt{\frac{2\log{\frac{1}{\mu}}}{KN}}.
 	\end{split}
 	\end{equation*}
 	Summing over $ T $ proves the theorem.
\end{proof}

\subsection{Network Specification}
\label{sec:network-specs}

All modules within our control mechanism share a similar network architecture that combines Long Short-Term Memory (LSTM)~\cite{hochreiter1997long} and Convolutional Neural Networks (CNN)~\cite{lecun2015deep}. Each image will first go through a CNN and then be combined with other images to form a training sample to go through a LSTM. The number of images of a training sample is empirically set to 5. We use the many-to-many mode of LSTM and set the number of hidden units of the LSTM to 100. The output is the average value of the output sequence.  

The CNN consists of eight layers. The first five are convolutional layers and the last three are dense layers. The kernel size is $ 5\times5 $ in the first three convolutional layers and $ 3\times3 $ in the other two convolutional layers. The first three convolutional layers have a stride of two while the last two convolutional layers are non-strided. The filters for the five convolutional layers are 24, 36, 48, 64, 64, respectively. All convolutional layers use VALID padding. The three dense layers have 100, 50, and 10 units, respectively. We use ELU as the activation function and $ \mathcal{L}2 $ as the kernel regularizer set to 0.001 for all layers. 


We train our model using Adam~\cite{kingma2014adam} with initial learning rate set to 0.0001. The batch size is 128 and the number of epochs is 500.
For training \textit{Detection} (a classification task), we use Softmax for generating the output and categorical cross entropy as the loss function.
For training \textit{Following} and \textit{Avoidance} (regression tasks), we use mean squared error (MSE) as the loss function. We have also adopted cross-validation with 90/10 split. The input image data have $ 220 \times 66 $ resolution in RGB channels.

\subsection{Example Expert Trajectories}
\label{sec:app-traj}
Figure~\ref{fig:trajectories} shows a set of generated trajectories for a situation where the vehicle had collided with a static obstacle in front of it after driving on a straight road.
As expected, the trajectories feature sharper turns (red trajectories) as the starting state tends towards the last moment that the vehicle can still avoid the obstacle. 

\begin{figure}
	\begin{center}
		\includegraphics[width=\columnwidth]{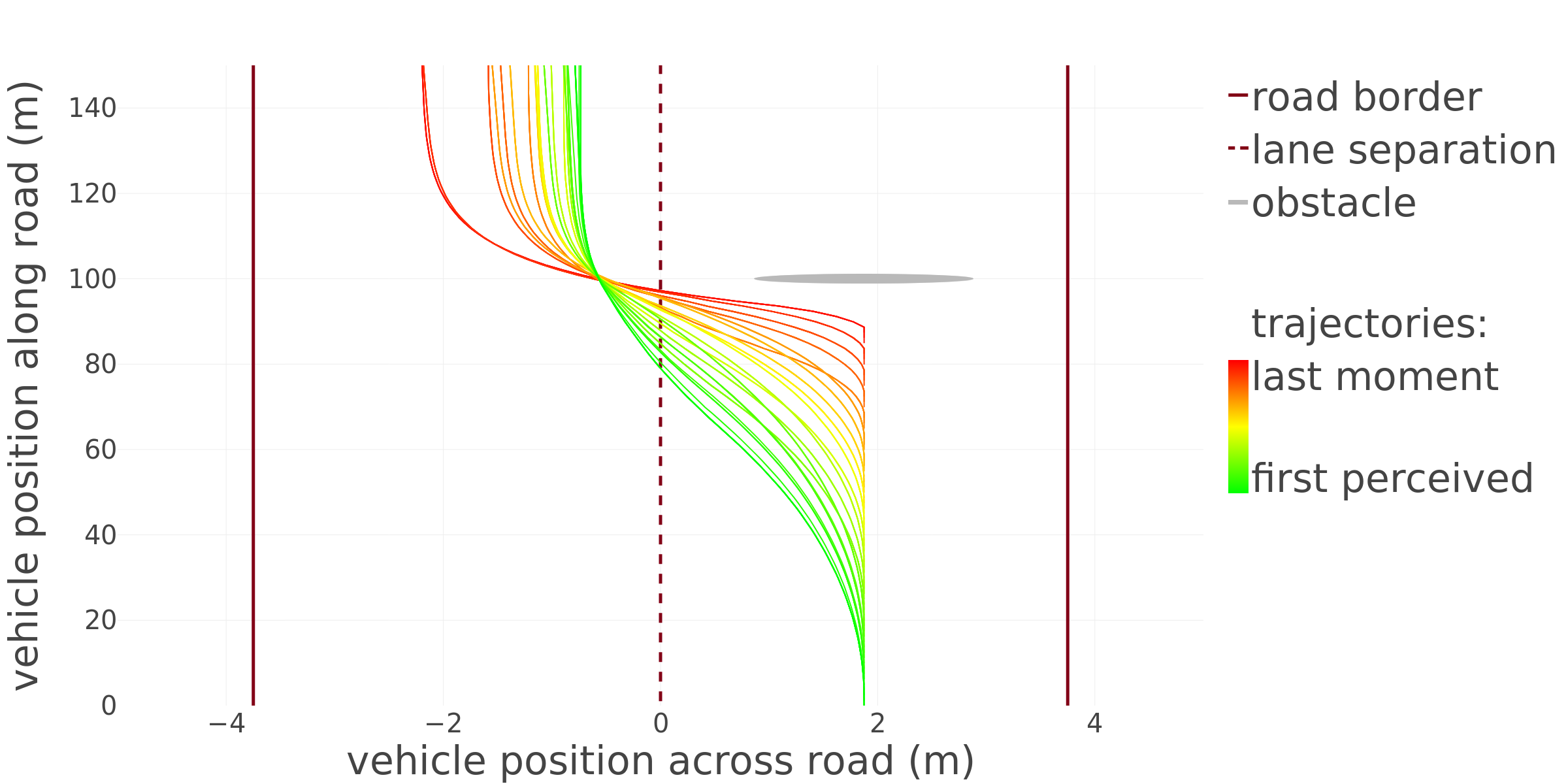}
	\end{center}
	\caption{
		Plotted collision-free trajectories generated by the expert algorithm for a vehicle traveling on the right lane of a straight road, with an obstacle in front.
		Spans 74 trajectories from the first moment the vehicle perceives the obstacle (green, progressive avoidance) to the last moment the collision can be avoided (red, sharp avoidance).
	}
	\label{fig:trajectories}
\end{figure}

\subsection{Learning From Accidents}
\label{sec:app-accidents}

\begin{figure}
	\begin{center}
		\includegraphics[width=1\columnwidth]{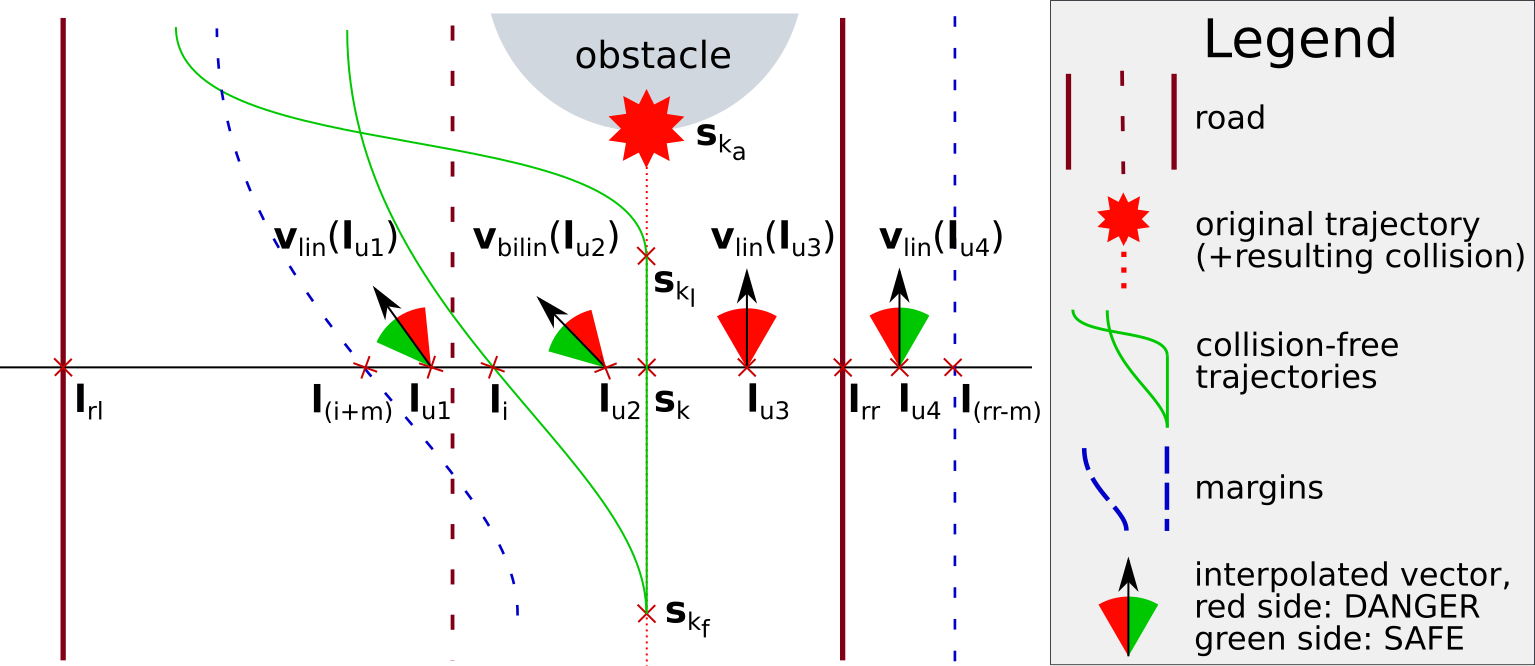}
	\end{center}
	\caption{
		(This figure is copied from the main text to here for completeness.) Illustration of important points and DANGER/SAFE labels from Section~\ref{sec:accident} for a vehicle traveling on the right lane of a straight road, with an obstacle in front.
		Labels are shown for four points $\{ \mathbf{l}_{u1}, \mathbf{l}_{u2}, \mathbf{l}_{u3}, \mathbf{l}_{u4} \}$ illustrating the four possible cases.
	}
	\label{fig:app-points}
\end{figure}

For the following paragraph, we abusively note $\mathbf{s}_k.x$, $\mathbf{s}_k.y$ the position coordinates at state $\mathbf{s}_k$, and $\mathbf{s}_k.vx$, $\mathbf{s}_k.vy$ the velocity vector coordinates at state $\mathbf{s}_k$.
Then, for any state $\mathbf{s}_k \in \{ \mathbf{s}_{k_f}, ..., \mathbf{s}_{k_a} \}$ we can define a line $L(\mathbf{s}_k) = \{ \mathbf{l}_u = \vectw{\mathbf{s}_k.x}{\mathbf{s}_k.y} + u \times \vectw{-\mathbf{s}_k.vy}{\mathbf{s}_k.vx} \: | \: u \in \mathbb{R} \}$.
On this line, we note $\mathbf{l}_i$ the furthest point on $L(\mathbf{s}_k)$ from $\vectw{\mathbf{s}_k.x}{\mathbf{s}_k.y}$ which is at an intersection between $L(\mathbf{s}_k)$ and a collision-free trajectory from $solve(\mathcal{S})$.
This point determines how far the vehicle can be expected to stray from the original trajectory $\mathcal{S}$ before the accident, if it followed an arbitrary trajectory from $solve(\mathcal{S})$.
We also note $\mathbf{l}_{rl}$ and $\mathbf{l}_{rr}$ the two intersections between $L(\mathbf{s}_k)$ and the road edges ($\mathbf{l}_{rl}$ is on the ``left'' with $rl > 0$, and $\mathbf{l}_{rr}$ is on the ``right'' with $rr < 0$).
These two points delimit how far from the original trajectory the vehicle could be.
Finally, we define a user-set margin $g$ as outlined below (we set $g=0.5~m$).

Altogether, these points and margin are the limits of the region along the original trajectory wherein we generate images for training: a point $\mathbf{l}_u \in L(\mathbf{s}_k)$ is inside the region if it is between the original trajectory and the furthest collision-free trajectory plus a margin $g$ (if $\mathbf{l}_u$ and $\mathbf{l}_i$ are on the same side, i.e. $sign(u) = sign(i)$), or if it is between the original trajectory and either road boundary plus a margin $g$ (if $\mathbf{l}_u$ and $\mathbf{l}_i$ are not on the same side, i.e. $sign(u) \neq sign(i)$).

In addition, if a point $\mathbf{l}_u \in L(\mathbf{s}_k)$ is positioned between two collision-free trajectories $\mathcal{\hat{S}}_{k_1}, \mathcal{\hat{S}}_{k_2} \in solve(\mathcal{S})$,
we consider the two closest states on $\mathcal{\hat{S}}_{k_1}$, and the two closest states from $\mathcal{\hat{S}}_{k_1}$, and bi-linearly interpolate these four states' velocity vectors, resulting in an approximate velocity vector $\mathbf{v}_{bilin}(\mathbf{l}_u)$ at $\mathbf{l}_u$.
Similarly, if a point $\mathbf{l}_u \in L(\mathbf{s}_k)$ is not positioned between two collision-free trajectories, we consider the two closest states on the single closest collision-free trajectory $\mathcal{\hat{S}}_{k_1} \in solve(\mathcal{S})$, and linearly interpolate their velocity vectors, resulting in an approximate velocity vector $\mathbf{v}_{lin}(\mathbf{l}_u)$ at $\mathbf{l}_u$.

From here, we can construct images at various points $\mathbf{l}_u$ along $L(\mathbf{s}_k)$ (increasing $u$ by steps of $0.1~m$), with various orientation vectors (noted $\mathbf{v}_u$ and within 2.5 degrees of $\vectw{\mathbf{s}_k.vx}{\mathbf{s}_k.vy}$), and label them using the following scheme (also illustrated in Figure~\ref{fig:app-points}).
If the expert algorithm made the vehicle avoid obstacles by steering left ($\mathbf{l}_i$ with $i > 0$), there are four cases to consider when building a point $\mathbf{l}_u$:
\begin{itemize}
	\item $u < i + g$ and $u > i$: $\mathbf{l}_u$ is outside of the computed collision-free trajectories $solve(\mathcal{S})$, on the outside of the steering computed by the expert algorithm. The label is SAFE if $det(\mathbf{v}_{lin}(\mathbf{l}_u), \mathbf{v}_u) \geq 0$, and DANGER otherwise.
	\item $u < i$ and $u > 0$: $\mathbf{l}_u$ is inside the computed collision-free trajectories $solve(\mathcal{S})$. The label is SAFE if $det(\mathbf{v}_{bilin}(\mathbf{l}_u), \mathbf{v}_u) \geq 0$ (over-steering), and DANGER otherwise (under-steering).
	\item $u < 0$ and $u > rr$: $\mathbf{l}_u$ is outside the computed collision-free trajectories $solve(\mathcal{S})$ on the inside of the steering computed by the expert algorithm. The label is always DANGER.
	\item $u < rr$ and $u > rr - g$: $\mathbf{l}_u$ is in an unattainable region, but we include it to prevent false reactions to similar (but safe) future situations. The label is DANGER if $det(\mathbf{v}_{lin}(\mathbf{l}_u), \mathbf{v}_u) > 0$, SAFE otherwise.
\end{itemize}

\noindent Here, the function $det(\cdot, \cdot)$ computes the determinant of two vectors from $\mathbb{R}^2$.

Conversely, if the expert algorithm made the vehicle avoid obstacles by steering right ($\mathbf{l}_i$ with $i < 0$), there are four cases to consider when building a point $\mathbf{l}_u$:
\begin{itemize}
	\item $u > i - g$ and $u < i$: the label is SAFE if $det(\mathbf{v}_{lin}(\mathbf{l}_u), \mathbf{v}_u) \leq 0$, and DANGER otherwise.
	\item $u > i$ and $u < 0$: the label is SAFE if $det(\mathbf{v}_{bilin}(\mathbf{l}_u), \mathbf{v}_u) \leq 0$, and DANGER otherwise.
	\item $u > 0$ and $u < rl$: the label is always DANGER.
	\item $u > rl$ and $u < rl + g$: the label is DANGER if $det(\mathbf{v}_{lin}(\mathbf{l}_u), \mathbf{v}_u) < 0$, SAFE otherwise.
\end{itemize}

We then generate images from these (position, orientation, label) triplets which are used to further train the \textit{Detection} module of our policy.

\subsection{Experiment Setup}
\label{sec:app-exp}


\subsubsection{Scenarios}
We have tested our method in three scenarios. The first is a straight road which represents a linear geometry, the second is a curved road which represents a non-linear geometry, and the third is an open ground. The first two represent on-road situations while the last represents an off-road situation. 

Both the straight and curved roads consist of two lanes. The width of each lane is $ 3.75~m $ and there is a $ 3~m $ shoulder on each side of the road. The curved road is half circular with radius at $ 50~m $ and is attached to two straight roads at each end. The open scenario is a $ 1000~m $ $ \times $ $ 1000~m $ ground, which has a green sphere treated as the target for the \textit{Following} module to steer the AV.


\subsubsection{Vehicle Specs}
The vehicle's speed is set to $ 20~m/s $, which value is used to compute the throttle value in the simulator.
Due to factors such as the rendering complexity and the delay of the communication module, the actual running speed is in the range of $ 20\pm1~m/s $.
The length and width of the vehicle are $ 4.5~m $ and $ 2.5~m $, respectively.
The distance between the rear axis and the rear of the vehicle is $ 0.75~m $.
The front wheels can turn up to 25 degrees in either direction.
We have three front-facing cameras set behind the main windshield, which are at $ 1.2~m $ height and $ 1~m $ front to the center of the vehicle.
The two side cameras (one at left and one at right) are set to be $ 0.8~m $ away from the vehicle's center axis.
These two cameras are only used to capture data for training \textit{Following}.
During runtime, our control policy only requires images from the center camera to operate.

\subsubsection{Obstacles}
For the on-road scenarios, we use a scaled version of a virtual traffic cone as the obstacle on both the straight and curved roads.
This scaling operation is meant to preserve the obstacle's visibility,
since at distances greater than $30~m$ a normal-sized obstacle is quickly reduced to just a few pixels.
This is an intrinsic limitation of the single-camera setup (and its resolution),
but in reality we can emulate this ``scaling'' using the camera's zoom function for instance. 
For the off-road scenario, we use a vehicle with the same specifications as of the AV as the dynamic obstacle. This vehicle is scripted to collide into the AV on its default course when no avoidance behavior is applied by the AV. 

\subsubsection{Training Data}
In order to train \textit{Following}, we have built a waypoint system on the straight road and curved road for the AV to follow, respectively. By running the vehicle for roughly equal distances on both roads, we have gathered in total 65~061 images (33~642 images for the straight road and 31~419 images for the curved road). On the open ground, we have sampled 30~000 positions and computed the angle difference between the direction towards the sphere target and the forward direction. This gives us 30~000 training examples.  

In order to train \textit{Avoidance}, on the straight road, we rewind the accident by 74 frames starting from the frame that the accident takes place, which gives us 74 safe trajectories.
On the curved road, we rewind the accident by 40  frames resulting in 40 safe trajectories.
On the open ground, we rewind the accident by 46 frames resulting in 46 safe trajectories.
By positioning the vehicle on these trajectories and capturing the image from the front-facing camera, we have collected 34~516 images for the straight road, 33~624 images for the curved road, and 33~741 images for the open ground. 

For the training of \textit{Detection}, using the mechanism explained in Subsection~\ref{sec:data_detect}, we have collected 32~538 images for the straight road, 71~859 for the curved road, and 67~102 images for the open ground.

\end{document}